\pdfoutput=1
\documentclass{article}

\usepackage{iclr2027_conference,times}
\iclrfinalcopy

\usepackage{amsmath,amsfonts,bm}

\def\eqref#1{equation~\ref{#1}}

\def\1{\bm{1}}

\def\eps{{\epsilon}}

\DeclareMathAlphabet{\mathsfit}{\encodingdefault}{\sfdefault}{m}{sl}
\SetMathAlphabet{\mathsfit}{bold}{\encodingdefault}{\sfdefault}{bx}{n}

\newcommand{\R}{\mathbb{R}}

\usepackage{hyperref}
\usepackage{url}
\usepackage{booktabs}
\usepackage{amsmath,amssymb,amsthm}
\usepackage{graphicx}
\usepackage{array}
\usepackage{xcolor}

\newtheorem{theorem}{Theorem}
\newtheorem{proposition}{Proposition}

\theoremstyle{remark}

\newcommand{\nz}{\mathcal{I}}
\newcommand{\IA}{I_{\mathrm{A}}}
\newcommand{\Ct}{C_t}

\title{CausalWAM: Learning the Causes of Actions in World Action Models}

\author{Juntao Cheng$^{1}$, Jinzhou Tang$^{2}$, Fan Feng$^{2}$, Henrik I.~Christensen$^{2}$, Yujia Zheng$^{1}$\\
\normalfont $^{1}$UIUC\quad $^{2}$UCSD}

\begin{document}
\maketitle

\begin{abstract}

World action models (WAMs) learn latent variables from observations and use these to both predict the evolution of the scene and to generate actions. The prediction has to be of the whole scene, whereas only a portion of the scene causes actions. WAMs fitted to demonstrations do not, in general, guarantee that they will generate actions solely from the latent causes of the demonstrated actions (the action-relevant latent variables), as the model could just as well learn to imitate in this way whilst entangling the action-relevant latent variables with the nuisance variables. We therefore separate out what causes action from what prediction depends on. We prove that under a sparsity constraint the latent variables used to generate actions are an invertible function of the action-relevant latent variables alone and the true structure is identified up to a permutation of the latent variables. This requires neither sufficient temporal dependence nor any assumptions on how the unobserved regimes evolve and holds even when the action-relevant latent variables are statistically dependent on the rest. We therefore propose CWAM (Causal WAM), which incorporates the constraint into a pretrained WAM by adding a sparse routing path to its action head, with the world modeling objective remaining unchanged. Experiments on both synthetic data and real-world robots validate the theory as well as showcase the improved efficiency in closed-loop success.

\end{abstract}

\section{Introduction}

The action a robot performs when stacking a single block atop another depends on the position of both blocks and the fact that the gripper is currently holding the uppermost block. The camera image's pixels will change significantly if the robot removes the tablecloth or clears away the clutter located behind the table; however, this will have no effect upon the motion itself. WAMs (World Action Models) like Fast-WAM \citep{yuan2026fastwam}, and Cosmos~3 \citep{nvidia2026cosmos3} are able to extract latent variables from images and can be used to generate actions from those latent variables, while also being capable of generating images of potential subsequent frames. Since prediction has to account for the whole scene, these latent variables also describe the tablecloth and the clutter. We refer to the latent variables that cause the demonstrated actions as action-relevant latent variables, and to the others as nuisance variables.

\looseness=-1
Fitting the demonstrations well does not in itself guarantee that the WAM is acting from the action-relevant latent variables. For each step of the demonstration we could write $o_t=f(s_t)$ and $a_t=\pi(s_t,\eta_t)$ where $s_t$ is the latent state, $f$ is some unknown invertible function, $\pi$ is some unknown action mechanism and $\eta_t$ is exogenous noise. Separating action from prediction is then about recovering (up to certain indeterminacies) the action-relevant latent variables and which of those variables cause each action dimension. Without further assumptions, the latent causal variables cannot be identified from their observations \citep{hyvarinen1999nonlinear,locatello2019challenging}. Observing the actions as well doesn't help with this ambiguity, as two models can give rise to the same distribution of $(o_t,a_t)$ whilst their latent variables differ by an arbitrary invertible transformation. This transformation can mix the action-relevant latent variables with the nuisance ones, so a very flexible action head could generate actions from variables which also contain information about the tablecloth (which doesn't in fact cause the motion) and still fit the demonstrations equally well.

To make latent variables identifiable, existing works in nonlinear independent component analysis (ICA) introduce auxiliary variables, such as domain or time indices, or exploit temporal dependence \citep{hyvarinen2016unsupervised,hyvarinen2017nonlinear,hyvarinen2019nonlinear,khemakhem2020variational}. For latent causal processes, identifiability has been established under sufficient temporal dependence \citep{yao2022temporally} and conditions on the latent dynamics \citep{song2023temporally}.  Interventions with known targets have also been introduced to pin down the latents \citep{lippe2022citris}. Many of these assumptions might be violated for demonstrations in complex real-world scenarios such as robotics, where some time steps might be independent, the latent dynamics might not be Markov, and controlled interventions are either prohibitively expensive or simply impossible.

Another line of research constrains the generative process by restricting the function class or enforcing sparsity in the dependency structure \citep{kivva2022identifiability,zheng2022identifiability,lachapelle2022disentanglement}. Most of these studies, similar to those previously discussed, aim to recover all latent variables, requiring their conditions to hold even for nuisance variables that a robot can disregard. In contrast, results on partial identifiability focus on a subset of latent variables, determined by side information accompanying the data, such as content shared across views \citep{vonkugelgen2021self,yao2024multiview}, components invariant across domains \citep{kong2022partial}, or elements relevant to specific tasks \citep{zheng2026generalist}. For robots learning from demonstrations, the relevant subset consists of the latent causes underlying the demonstrated actions. Since no view, domain, or task label directly specifies this subset, the actions themselves must indicate which latent variables are pertinent. This raises the following question:
\begin{center}
\emph{From demonstrations alone, can we separate what causes actions from what prediction needs?}
\end{center}
In this work, we answer this question with a constraint on the action side that leaves prediction free to use every latent variable. Since each action is generated from only part of the latent state, constraining the number of dependencies between latent variables and action dimensions can separate that part from the rest. Specifically, we prove that, without such a constraint and under nondegeneracy conditions, fitting the demonstrations only ensures that each action dimension of the estimated mechanism depends on at least as many latent variables as the true one (Prop. \ref{prop:generalist}). We then incorporate a sparsity regularization on the dependency structure between latent state and action variables during estimation. Under the same conditions, the lower bound becomes an equality and the dependency structure is identified up to a permutation of the latent variables (Thm. \ref{thm:action}).

Moreover, we prove that the estimated action-relevant variables under the proposed constraint are an invertible function of the ground-truth latent variables causing the true action (Thm. \ref{thm:state}). The latent causes of the actions are thus identified as a block and are separated from the nuisance variables. The theoretical guarantee relies purely on observational demonstrations, without requiring any interventions, sufficient temporal dependence, or any assumptions about the latent dynamics. The framework also allows dependence between action-relevant latent variables and the nuisance variables (e.g., a tablecloth could be correlated with the positions of the cubes). According to the established theory, we propose CWAM (Causal WAM), which adds a sparse routing path from the visual and task tokens to the action head in Cosmos~3 and applies a token-level surrogate of the constraint to this path (Sec. \ref{sec:method}). In synthetic scenarios, we validate the identifiability claims shown by our theory. On RoboCasa365, LIBERO, and a real Franka R3 robot, the routing path improves aggregate closed-loop success over Cosmos~3 without it, and in most visualized frames its routing mass falls on the object being moved and its target (Sec. \ref{sec:experiments}).

\section{Problem setup}
\label{sec:setup}

\paragraph{Data.} Demonstration data are sequences $\{o_t,a_t\}_{t=1}^T$ with $o_t\in\R^{d_o}$ and $a_t\in\R^{d_a}$. Relevance is defined by the executed action at each step, and no task, skill, or segment labels are assumed.

\paragraph{Observational world action process.} We define latent states $\{s_t\}_{t=1}^T$, $s_t\in\R^n$, and hidden nonparametric maps
\begin{subequations}\label{eq:dgp}
\begin{align}
  o_t &= f(s_t),\label{eq:dgp-obs}\\
  a_t &= \pi(s_t,\eta_t),\label{eq:dgp-action}
\end{align}
\end{subequations}
where $f:\R^n\to\R^{d_o}$ is a diffeomorphism onto its image ($d_o\gg n$ is allowed), $\pi$ is an unknown action mechanism, and $\eta_t$ is exogenous noise.
The same
mechanism applies to an $H$-step chunk $a_{t:t+H-1}$ when the policy emits a short
horizon at once (Section~\ref{sec:method}); identifiability is stated per action
coordinate.
We do not require temporal dependence: consecutive steps may be disconnected and each $(o_t,a_t)$ may be i.i.d. Within a step, action-relevant latent variables are those on which $\pi$ functionally depends. Since $f$ is invertible, $o_t$ depends on every latent variable, whereas $\pi$ depends only on the action-relevant ones. A fitted action mechanism may \emph{over-cover} the support of $J_\pi$, that is, depend on more latent variables than $\pi$, and we ask when the latent variables used to generate actions are separated from the nuisance ones (notation below).

\paragraph{Mechanism supports and the compact state.} For a differentiable structural map $m$ of the state, let $J_m(s)=\partial m/\partial s$, and let $\nz(J_m)=\{(i,j):\Pr[J_m(s)_{ij}\neq0]>0\}$ be its functional support, with $\|\nz(J_m)\|_0$ its cardinality. For $\pi$, the Jacobian is taken with respect to $s_t$ at fixed $\eta_t$. We write $I_i:=\{j:(i,j)\in\nz(J_\pi)\}$ for the coordinates that drive action dimension~$i$, set $\IA=\bigcup_i I_i$, and write $s_{t,\IA}=(s_{t,j})_{j\in\IA}$ for the action-relevant subvector. Relevance is functional (a perturbation of coordinate $j$ can change $a$) and is not a statement about correlation. A learner returns an encoder $\hat f^{-1}$ with $\hat s_t=\hat f^{-1}(o_t)$, an action mechanism $\hat\pi$, estimated index sets $\hat I_i:=\{j:(i,j)\in\nz(J_{\hat\pi})\}$ and $\hat{I}_{\mathrm{A}}=\bigcup_i\hat I_i$, and the \emph{compact action-relevant state} $\Ct:=\hat s_{t,\hat{I}_{\mathrm{A}}}$. At deployment, $\Ct$ is computed from observations only; training may use $a_t$ to fit $\hat\pi$, but actions are not inputs to $\hat f^{-1}$.

\paragraph{Identification goal.} Write $\mathcal{L}_{f,\pi}$ for the law of $(o,a)$ induced by \eqref{eq:dgp-obs}--\eqref{eq:dgp-action}. Two processes $(f,\pi)$ and $(\hat f,\hat\pi)$ are \emph{observationally equivalent} if
\begin{equation}\label{def:obs-eq}
  \mathcal{L}_{f,\pi}
  \;=\;
  \mathcal{L}_{\hat f,\hat\pi}.
\end{equation}
Under observational equivalence on $(o,a)$, the estimated state $\hat s_t$ determines every latent variable, as prediction requires, so we only ask about the action side. We first ask whether the estimated action mechanism matches the true state-to-action dependency pattern. We then ask whether the compact state $\Ct$ used to generate actions is separated from the nuisance variables, in the sense that it carries the action-relevant block $s_{t,\IA}$ up to a block-wise invertible transformation. Section~\ref{sec:theory} states sufficient conditions for both questions. Section~\ref{sec:exp-synthetic} is the direct empirical check with known $\IA$, and Appendix~\ref{app:protocol} records the synthetic generator, estimators, and metrics in full.

\paragraph{Estimated process.} An estimated process is any pair $(\hat f,\hat\pi)$ that satisfies \eqref{eq:dgp-obs}--\eqref{eq:dgp-action} with $\hat f:\R^n\to\R^{d_o}$ a diffeomorphism onto its image. With $\hat s_t=\hat f^{-1}(o_t)$ and $\phi:=\hat f^{-1}\circ f$, observational equivalence allows $\phi$ to be an arbitrary diffeomorphism on $\R^n$ compatible with \eqref{def:obs-eq}. By itself, observational equivalence does not identify the full latent state $s_t$ or a unique encoder $\hat f^{-1}$. Section~\ref{sec:theory} asks what \emph{is} identified (the compact state $\Ct$ and the functional support of $J_\pi$) under additional inductive bias.

\paragraph{Regularity and coupling.} Throughout, $f$, $\hat f$, $\pi(\cdot,\eta)$, and $\hat\pi(\cdot,\eta)$ are continuously differentiable, $f$ and $\hat f$ have the same image (so $\phi$ is a diffeomorphism of $\R^n$), and $s_t$ has a density that is positive on $\R^n$ (hence so does $\hat s_t=\phi(s_t)$). The results below use a pathwise coupling of the two processes, in which they share the noise and share the latent state up to $\phi$, that is, $o_t=f(s_t)=\hat f(\hat s_t)$ and $a_t=\pi(s_t,\eta_t)=\hat\pi(\hat s_t,\eta_t)$ with $\hat s_t=\phi(s_t)$. This coupling implies \eqref{def:obs-eq}.

\vspace{-0.2em}
\section{Theory}
\label{sec:theory}
\vspace{-0.2em}

We now study, within each time step, when fitting $(o,a)$ separates the latent causes of the action from the nuisance coordinates that only prediction needs. In \eqref{eq:dgp}, only the block $s_{t,\IA}$ need enter $\pi$. Remaining coordinates may still be required for $o_t=f(s_t)$, so a latent state sufficient for observations need not isolate the coordinates that cause the action. All guarantees below concern one application of \eqref{eq:dgp}, or equivalently one pair $(o,a)$ drawn from $\mathcal{L}_{f,\pi}$ in \eqref{def:obs-eq}. We keep the subscript $t$ when referring to demonstration tuples $(o_t,a_t)$, but the argument does not use cross-time structure (Appendix~\ref{app:proof} fixes a step and drops $t$). The within-step conclusion applies at every $t$, and if the steps are identically distributed, the law in \eqref{def:obs-eq} also does not depend on $t$.

\subsection{Identifiability without sparsity}
\label{sec:theory-nosparse}

Without a support-sparsity constraint on $\hat\pi$, observational equivalence on $(o,a)$ may still hold while $\hat\pi$ depends on more latent coordinates than $\pi$. The action side then need not be separated from the nuisance coordinates, and fitting guarantees only the one-sided bound in Proposition~\ref{prop:generalist}.

\begin{proposition}\label{prop:generalist}
Consider two observationally equivalent processes $o_t=f(s_t)=\hat f(\hat s_t)$, $a_t=\pi(s_t,\eta_t)=\hat\pi(\hat s_t,\eta_t)$ of the form \eqref{eq:dgp} that satisfy the regularity conditions of Section~\ref{sec:setup}. Assume that, for each action dimension $i\in[d_a]$, there exist $|I_i|$ states $s^{(1)},\dots,s^{(|I_i|)}$ at which the rows $(J_\pi(s^{(l)}))_{i,\cdot}$ are linearly independent, and for $\phi=\hat f^{-1}\circ f$, $P=J_{\phi^{-1}}$, and a matrix $M$ with $\nz(M)=\nz(P)$, $\nz\big((J_\pi(s^{(l)})M)_{i,\cdot}\big)\subseteq\nz\big((J_{\hat\pi}(\hat s^{(l)}))_{i,\cdot}\big)$ with $\hat s^{(l)}=\phi(s^{(l)})$. Then, for every action dimension $i\in[d_a]$,
\[
\|\nz\big((J_{\hat\pi})_{i,\cdot}\big)\|_0\;\geq\;\|\nz\big((J_\pi)_{i,\cdot}\big)\|_0 .
\]
\end{proposition}

\paragraph{Proof sketch.} The span condition ensures that reparameterization between $s$ and $\hat s$ cannot eliminate a true functional dependency on $I_i$. The inclusion condition links $J_{\hat\pi}$ to $J_\pi$ through $\phi$ and rules out accidental cancellations, so, up to a permutation $\rho$ of the latent coordinates that is common to all action dimensions, each row of $J_{\hat\pi}$ must activate at least the coordinates required by the corresponding row of $J_\pi$. Counting these coordinates gives the bound. Details are in Appendix~\ref{app:proof}.

\paragraph{Discussion on assumptions.}
The span condition in Proposition~\ref{prop:generalist}
requires the Jacobian row of each action dimension to span
the coordinate subspace indexed by its true parents.
This is a separate assumption on the action mechanism:
smoothness and a positive state density alone do not imply it.
A fuller discussion is in Appendix~\ref{app:nondeg}.

\paragraph{Implication.} Proposition~\ref{prop:generalist} makes over-coverage precise. Without a bias toward a \emph{minimal} action pathway, observational fitting need not shrink $\hat\pi$ to $\pi$. It need only learn a mechanism that is no smaller in support than the true one. The estimated structure from latent coordinates to actions can then contain spurious edges, through which nuisance coordinates enter the action side. Neither matching $\nz(J_\pi)$ nor invertible recovery of $s_{t,\IA}$ in $\Ct$ is guaranteed.

\subsection{Identifiability under support sparsity}
\label{sec:theory-sparse}
The negative result points to the design choice used in our method: constrain \emph{which} latent coordinates $\hat\pi$ may depend on. Support sparsity requires that the estimated mechanism activate no more coordinates, counted over action dimensions, than the true one,
\begin{equation}\label{eq:sparsity}
\|\nz(J_{\hat\pi})\|_0\;\leq\;\|\nz(J_\pi)\|_0 .
\end{equation}
Combined with Proposition~\ref{prop:generalist}, this inequality bounds the size of the estimated support from both sides and identifies the state-to-action dependency pattern up to relabeling. The mechanism-level guarantee is as follows.

\begin{theorem}\label{thm:action}
Consider two observationally equivalent processes $o_t=f(s_t)=\hat f(\hat s_t)$,
$a_t=\pi(s_t,\eta_t)=\hat\pi(\hat s_t,\eta_t)$ of the form \eqref{eq:dgp}, and assume
the regularity conditions of Section~\ref{sec:setup}, the conditions in Proposition~\ref{prop:generalist}, and \eqref{eq:sparsity}.
Then there is a permutation $\rho$ of $[n]$ such that
\[
\nz(J_{\hat\pi})\;=\;\nz(J_\pi)\,P_\rho,
\]
and, for every action dimension $i\in[d_a]$,
$\nz\big((J_{\hat\pi})_{i,\cdot}\big)=\rho(I_i)$.
Here $P_\rho$ is the permutation matrix with $(P_\rho)_{j,k}=\mathbf 1\{\rho(j)=k\}$, and supports are identified with $0/1$ matrices.
\end{theorem}

\paragraph{Proof sketch.} By the proof of Proposition~\ref{prop:generalist}, reparameterization through $\phi=\hat f^{-1}\circ f$ gives a common permutation $\rho$ with $\rho(I_i)\subseteq\nz((J_{\hat\pi})_{i,\cdot})$ for every $i$. Sparsity \eqref{eq:sparsity} forbids extra activations, so the inclusions become equalities and assemble into $\nz(J_{\hat\pi})=\nz(J_\pi)P_\rho$ (Appendix~\ref{app:proof}).

\paragraph{Implication.} Each action dimension $i$ of the fitted mechanism depends on exactly the estimated coordinates $\rho(I_i)$, which correspond to its true parents $I_i$, and the bipartite pattern from latent coordinates to actions is matched, whereas without sparsity it can be enlarged. Sparsity turns the lower bound in Proposition~\ref{prop:generalist} into an exact answer to \emph{which} coordinates each action dimension functionally depends on. Theorem~\ref{thm:action} concerns the Jacobian of $\hat\pi$, that is, the bipartite pattern from latent coordinates to executed actions. However, the robot deploys the compact state $\Ct=\hat s_{t,\hat{I}_{\mathrm{A}}}$, which is computed from observations without access to $a_t$. Even with matched supports, one could imagine a lossy projection that drops information within the action-relevant block or leaks nuisance coordinates into $\Ct$. The next result shows that this does not happen under the coupling of Section~\ref{sec:setup} once supports align. Specifically, the reparameterization $\phi=\hat f^{-1}\circ f$ cannot mix the nuisance coordinates into the estimated coordinates $\rho(\IA)$, so $\Ct$ is an invertible function of $s_{t,\IA}$.

\begin{theorem}\label{thm:state}
Under the hypotheses of Theorem~\ref{thm:action}, including the regularity conditions of Section~\ref{sec:setup}, with the same permutation $\rho$,
there exists an invertible map $h$ such that
\[
\Ct\;=\;\hat s_{t,\rho(\IA)}\;=\;h\big(s_{t,\IA}\big).
\]
\end{theorem}

\paragraph{Proof sketch.} For $\phi=\hat f^{-1}\circ f$ and $P=J_{\phi^{-1}}$, the aligned row supports $\nz((J_{\hat\pi})_{i,\cdot})=\rho(I_i)$ from Theorem~\ref{thm:action} imply that $P$ is block lower-triangular, with a zero block in rows $\IA$ and columns $\rho([n]\setminus\IA)$. The invertibility of $P$ then yields $\hat s_{t,\rho(\IA)}=h(s_{t,\IA})$ with invertible $h$. By Theorem~\ref{thm:action}, $\hat I_i=\rho(I_i)$ for every $i$, so $\hat{I}_{\mathrm{A}}=\rho(\IA)$ and $\Ct=\hat s_{t,\rho(\IA)}$. Details are in Appendix~\ref{app:proof}.

\paragraph{Implication.} The compact state $\Ct$ is an invertible re-encoding of the true action-relevant block $s_{t,\IA}$. This is stronger than being correlated with pose or gripper-like factors. The conclusion is functional ($\Ct$ depends only on $s_{t,\IA}$ and determines it), so it holds whatever the statistical dependence between $s_{t,\IA}$ and the nuisance coordinates.

\vspace{-0.5em}
\section{Estimation method}
\label{sec:method}
\vspace{-0.5em}

\paragraph{Estimation principle.} Under the pathwise coupling and regularity conditions of Section~\ref{sec:setup}, the nondegeneracy conditions of Proposition~\ref{prop:generalist}, and support sparsity \eqref{eq:sparsity}, Section~\ref{sec:theory} identifies the support of $J_\pi$ up to a permutation and the action-relevant latent variables in $\Ct$ up to a block-wise invertible transformation. We refer to a model fitted to demonstrations with a sparsity penalty on its action mechanism as CWAM (Causal WAM). Synthetic experiments impose coordinate sparsity on a linear head directly (Section~\ref{sec:exp-synthetic}). The deployed action interface extends \eqref{eq:dgp-action} to an $H$-step chunk:
\begin{equation}\label{eq:method-goal}
  \hat a_{t:t+H-1}
  \;=\;
  \hat\pi\bigl(o_t,\,\epsilon_t;\,\Ct\bigr),
  \qquad
  \Ct \;=\; \mathcal{S}\bigl(\hat f^{-1}(o_t)\bigr),
\end{equation}
where $\epsilon_t$ is exogenous noise for generation (an initial draw for rectified flow, in parallel with $\eta_t$ in \eqref{eq:dgp-action}), $\mathcal{S}$ selects a support-sparse subset of the estimated state, and $H$ is the chunk horizon. Figure~\ref{fig:overview} shows the split between world modeling and sparse routing to the action head.

\begin{figure}[t]
  \centering
  \includegraphics[width=\linewidth]{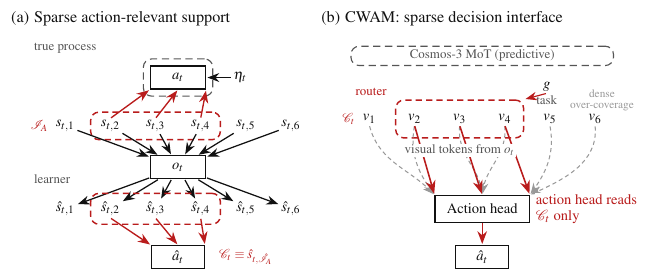}
  \vspace{-2em}
  \caption{The action depends only on a block of the latent state, and CWAM routes an estimate of this block to the action head through a sparse path. (a) In the world action process \eqref{eq:dgp} this block is $s_{t,\IA}$, and the learner $(\hat f,\hat\pi)$ estimates it from observations as the compact state $\Ct$. (b) CWAM keeps the joint rectified-flow training of Cosmos~3 and adds a sparse router over clean context tokens, whose top-$k$ routed tokens form $\Ct$. Gray dashed arrows mark the dense path to the action head, which the default policy keeps.}
  \label{fig:overview}
\end{figure}

\paragraph{Backbone.} Cosmos~3~Nano \citep{nvidia2026cosmos3} co-trains video prediction and action generation from shared context. CWAM keeps that backbone and post-trains it on demonstrations $(o,a)$ with the standard joint rectified-flow recipe \citep{lipman2023flow,liu2023rectified}. A mixture-of-transformers (MoT) generator~\citep{liang2025mot} denoises future video latents and the executed action chunk, conditioned on clean history and language. The leading frame in each window is observed, and later video tokens are noised during training. CWAM instantiates the estimation principle above as
\begin{equation}\label{eq:estimation-objective}
  \min_{\hat f,\hat\pi}\;
  \mathcal{D}\bigl(\mathcal{L}_{f,\pi},\,\mathcal{L}_{\hat f,\hat\pi}\bigr)
  \;+\;
  \mathcal{R}_{\mathrm{sparse}}(\hat\pi),
\end{equation}
where $\mathcal{D}$ penalizes mismatch between the true and estimated laws of $(o,a)$ from \eqref{eq:dgp-obs}--\eqref{eq:dgp-action}, and $\mathcal{R}_{\mathrm{sparse}}$ biases $\hat\pi$ toward a minimal action pathway as in \eqref{eq:sparsity}. In Cosmos~3, $\mathcal{D}$ is the joint post-training loss $\mathcal{L}_{\mathrm{WAM}}$ (the rectified-flow video and action co-training described above), and $\mathcal{R}_{\mathrm{sparse}}=\mathcal{L}_{\mathrm{sparse}}$ in \eqref{eq:router-sparse}. CWAM leaves $\mathcal{L}_{\mathrm{WAM}}$ unchanged, adds the sparse router, and applies $\mathcal{L}_{\mathrm{sparse}}$ to the router readout only.

\paragraph{Action-relevant router.} After the generator tower, each action hidden state $h_i$ routes over clean context indices $j\in\mathcal{J}_t$ (task tokens and \emph{clean} vision, excluding future noised tokens):
\begin{equation}\label{eq:router-attn}
  w_{ij}
  \;=\;
  \mathrm{softmax}_{j\in\mathcal{J}_t}\Bigl(
  \frac{q_i^\top k_j}{\sqrt{d_r}} + b_{ij}(\text{relations})
  \Bigr),
  \qquad
  \tilde h_i \;=\; \sum_{j\in\mathcal{J}_t} w_{ij}\, v_j,
\end{equation}
with relation-aware biases $b_{ij}$ on position, modality, and chunk progress ($d_r=512$, 8 heads). The routed signal enters through a learned gate $g\in(0,1)$, initialized small, as $h_i\leftarrow h_i + g\,\tilde h_i$. Training adds a token-level surrogate for \eqref{eq:sparsity} on this pathway,
\begin{equation}\label{eq:router-sparse}
  \mathcal{L}_{\mathrm{sparse}}
  \;=\;
  \lambda\,\mathbb{E}_i\Bigl[1-\sum_{j\in T_k(i)}w_{ij}\Bigr],
\end{equation}
where $T_k(i)\subseteq\mathcal{J}_t$ is the index set of the $k$ largest weights $w_{ij}$, so the sum is the attention mass on the top-$k$ context tokens ($\lambda=0.01$, ramped in after $10^3$ steps). Post-training minimizes
\begin{equation}\label{eq:cwam-train}
  \mathcal{L}_{\mathrm{total}}
  \;=\;
  \mathcal{L}_{\mathrm{WAM}}
  \;+\;
  \mathcal{L}_{\mathrm{sparse}},
\end{equation}
with $\mathcal{L}_{\mathrm{WAM}}$ identical to that of the matched Cosmos~3 baseline on the same demonstration windows (rectified-flow targets for noised future video tokens and the action chunk).

\paragraph{Compact state.} The interface in \eqref{eq:method-goal} idealizes $\Ct=\mathcal{S}(\hat f^{-1}(o_t))$ on encoder coordinates; on Cosmos~3 the same role is played by clean context tokens:
\begin{equation}\label{eq:ct-tokens}
  \Ct
  \;=\;
  \Bigl\{
  j\in\mathcal{J}_t \;\Big|\;
  \exists\,i:\; j\in T_k(i)
  \Bigr\},
\end{equation}
i.e.\ clean tokens in the top-$k$ routed mass of some action query. The default policy retains the dense MoT visual to  action path, and \eqref{eq:router-sparse} regularizes the router branch of $\hat\pi$.

\vspace{-0.2em}
\section{Experiments}
\label{sec:experiments}
\vspace{-0.2em}

We evaluate CWAM in two settings. On synthetic world action processes with known latent variables, the compact state $\Ct$ used by the sparse action head recovers the action-relevant block, which is consistent with Theorem~\ref{thm:state} (Section~\ref{sec:exp-synthetic}). On RoboCasa365, LIBERO, and a real Franka robot, the sparse routing path improves aggregate closed-loop success over the matched Cosmos~3 baseline (Section~\ref{sec:exp-main-learned}), and on the four RoboCasa365 tasks where CWAM is strongest, the routed tokens alone keep most of that success (Section~\ref{sec:exp-ablations}). Appendix~\ref{app:protocol} gives the synthetic protocol in full.

\subsection{The compact state recovers the action-relevant block}
\label{sec:exp-synthetic}

Theorem~\ref{thm:state} predicts that $\Ct$ is an invertible function of the action-relevant block $s_{t,\IA}$ alone. Comparing $\Ct$ with this prediction needs the true latent variables and $\IA$, which a simulator provides.

\paragraph{Setup.} The process follows \eqref{eq:dgp} with $n=20$ latent variables, $|\IA|=4$ of which drive the action, and observes each latent variable through its own noisy coordinates. A hidden regime, never given to the estimator, scales the action mechanism and a transition that couples the nuisance variables to the action-relevant ones. Each estimator pairs an observation-only MLP encoder with a linear action head, whose weight matrix is its Jacobian, so CWAM with sparsity uses a group penalty on the head columns, a column-level surrogate for \eqref{eq:sparsity}. The compact state $\Ct$ keeps the $|\IA|$ coordinates with the largest mean absolute head weight. The baselines are a standard MLP VAE \citep{zheng2026generalist} fitted without head sparsity and three readouts of one fit without the penalty, namely the full latent state (dense), $|\IA|$ random coordinates (random matched), and the selected coordinates (CWAM without sparsity). Held-out linear probes give the $R^2$ from $\Ct$ to the action-relevant latent variables. For the nuisance variables we report $\Delta_{\mathrm{nuis}}=R^2(\Ct\!\to\!\mathrm{nuis})-R^2(s_{t,\IA}\!\to\!\mathrm{nuis})$, the excess of $\Ct$ over the oracle reference, i.e., over how well the true block predicts them.

\begin{table}[t]
  \caption{The coordinates selected by CWAM with sparsity decode the action-relevant latent variables as well as the dense readout, and the nuisance variables only about as well as the true block. Zeroing them breaks the fitted head. Held-out linear probes give the relevant $R^2$ and $\Delta_{\mathrm{nuis}}$, the nuisance $R^2$ of $\Ct$ minus that of $s_{t,\IA}$. Selected/full (Complement/full) is the held-out action MSE of the fitted head, not refitted, after zeroing the coordinates outside (inside) $\Ct$, relative to the full $\hat s_t$. Mean $\pm$ population std over five seeds, with ratios computed per seed.}
  \label{tab:synthetic}
  \centering
  {\footnotesize\setlength{\tabcolsep}{2.5pt}%
  \resizebox{0.92\linewidth}{!}{%
  
\begin{tabular}{@{}p{0.26\linewidth}ccccc@{}}
\toprule
Method & $|\Ct|$ & Relevant $R^2\uparrow$ & $\Delta_{\mathrm{nuis}}\downarrow$ & Selected/full & Complement/full \\
\midrule
Dense (20D readout) & 20 & 0.9961 $\pm$ 0.0015 & 0.646 $\pm$ 0.115 & n/a & n/a \\
Random matched (4D) & 4 & 0.5493 $\pm$ 0.2964 & 0.392 $\pm$ 0.144 & 48.067 $\pm$ 31.855 & 6.36 $\pm$ 6.42 \\
CWAM without sparsity & 4 & 0.9876 $\pm$ 0.0164 & -0.008 $\pm$ 0.006 & 3.456 $\pm$ 3.180 & 50.58 $\pm$ 28.81 \\
\midrule
Standard MLP VAE & 4 & 0.9440 $\pm$ 0.0448 & 0.316 $\pm$ 0.186 & 3.259 $\pm$ 2.211 & 51.13 $\pm$ 28.54 \\
CWAM with sparsity & 4 & 0.9957 $\pm$ 0.0014 & 0.039 $\pm$ 0.093 & 1.704 $\pm$ 1.398 & 54.65 $\pm$ 25.78 \\
\bottomrule
\end{tabular}
  }}
\end{table}

\paragraph{Block recovery.} A block that depends on $s_{t,\IA}$ alone can predict the nuisance variables only through their dependence on $s_{t,\IA}$, so under Theorem~\ref{thm:state} the nuisance excess of $\Ct$ should be near zero. For the coordinates selected by CWAM with sparsity it is indeed close to zero, and they decode the action-relevant latent variables about as well as the dense readout (Table~\ref{tab:synthetic}). In contrast, the dense readout and the coordinates selected from the VAE decode the nuisance variables well beyond the true block, and four random coordinates recover the action-relevant ones poorly.

\paragraph{Head ablation.} A probe measures what $\Ct$ contains, while the fitted head may still draw on other coordinates. Setting $\Ct$ to zero raises the action error by a similarly large factor for every head that selected a block, including the VAE, while setting the complement to zero costs much less (Table~\ref{tab:synthetic}). Random matched shows the opposite pattern, and keeping only its coordinates breaks the head, so the head depends on which coordinates are kept more than on how many.

\paragraph{Hidden environments.} The theorems assume nothing about how unobserved regimes evolve, so hidden variation in the nuisance dynamics should leave the conclusion of Theorem~\ref{thm:state} unchanged. A second hidden index $e_t$, resampled at every step and never given to the encoder, rescales only the nuisance variables in the transition, and we increase its number of values $E$ with the action mechanism fixed. As $E$ grows, neither readout loses accuracy on the action-relevant latent variables, and only the dense readout keeps a large nuisance excess (Appendix~\ref{app:synthetic-env}, Table~\ref{tab:synthetic-env}).

\subsection{The sparse routing path improves closed-loop success}
\label{sec:exp-main-learned}

For a pretrained WAM, the true latent variables are unknown, so on Cosmos~3 we ask whether the sparse routing path improves closed-loop success. We compare with Cosmos~3 which is post-trained identically with matched setups but without the router and $\mathcal{L}_{\mathrm{sparse}}$ (Section~\ref{sec:method}). CWAM succeeds more often on RoboCasa365 \citep{nasiriany2026robocasa365}, LIBERO, and a real Franka robot that stacks cubes (Figure~\ref{fig:robocasa-libero-radar}).

\begin{figure}[t]
  \centering
  \includegraphics[width=\linewidth]{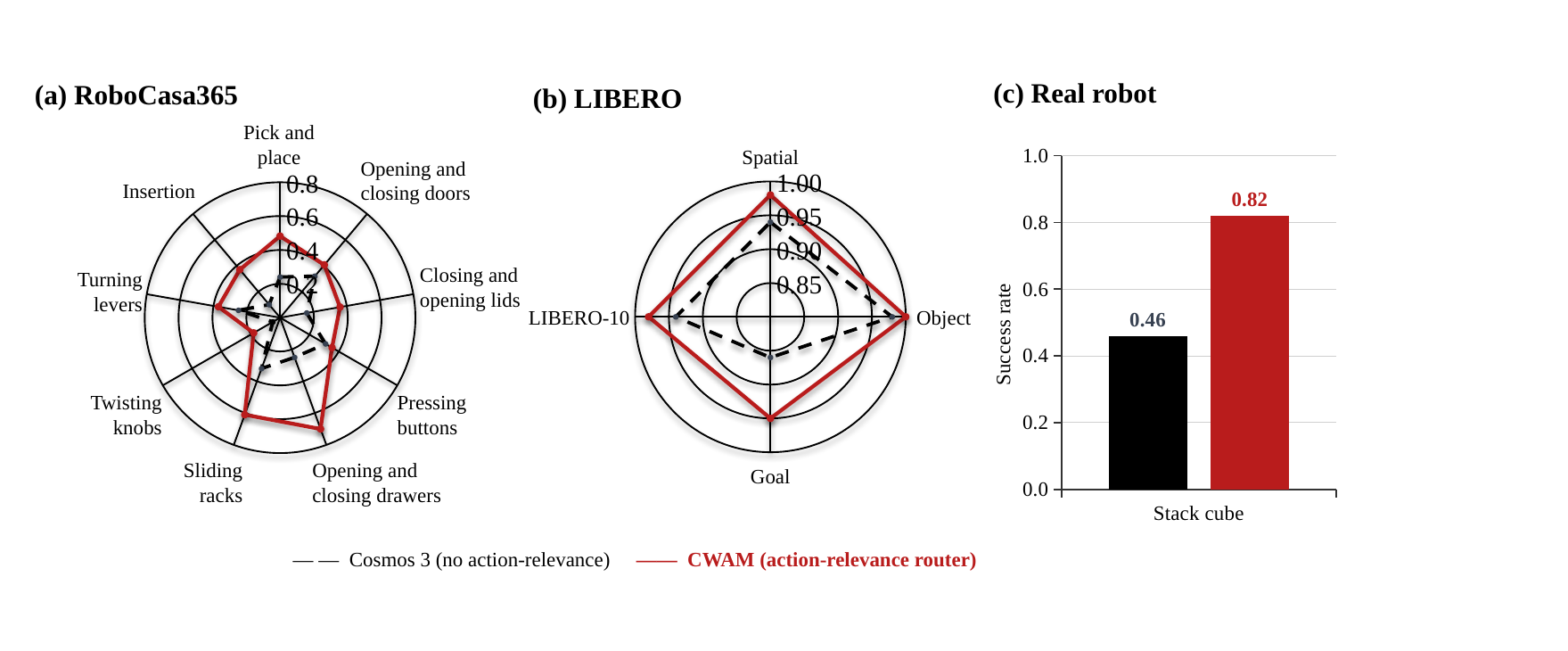}
  \vspace{-3em}
  \caption{CWAM improves closed-loop success over the matched Cosmos~3 baseline in every RoboCasa365 skill family, every LIBERO suite, and on the real robot. (a)~RoboCasa365, mean success per skill family over $17$ tasks (tasks in Table~\ref{tab:robocasa-per-task}). (b)~LIBERO, success per suite over $40$ tasks. (c)~Franka stack cube, $50$ trials per policy. Panels (a) and (b) use $100$ episodes per task. The baseline (Cosmos~3 in the legend) is post-trained on the same data and recipe, without the router and $\mathcal{L}_{\mathrm{sparse}}$.}
  \label{fig:robocasa-libero-radar}
\end{figure}

On RoboCasa365, CWAM achieves a higher mean success rate than the baseline in every foundational skill family (Figure~\ref{fig:robocasa-libero-radar}(a)). The gain is small for button pressing and largest for drawer opening and closing, rack sliding, insertion, and pick-and-place tasks. Averaged across tasks, CWAM's success rate is nearly twice that of the baseline (Table~\ref{tab:robocasa-per-task}). On LIBERO \citep{liu2023libero}, both policies achieve high success rates, with CWAM outperforming the baseline across all suites and showing the largest gain on LIBERO-Goal (Figure~\ref{fig:robocasa-libero-radar}(b)). In real-robot experiments, CWAM succeeds in a majority of trials, whereas the baseline succeeds in fewer than half (Figure~\ref{fig:robocasa-libero-radar}(c)).

\paragraph{Where does the router look?}
In sampled frames of the real-robot sequence in Figure~\ref{fig:router-mass-stackcube} (bottom), routing weights are concentrated on the two cubes involved in the stacking task described in the introduction. The RoboCasa365 rollout (top) shows a similar pattern, with most of the weight assigned to the manipulated object and its target. In some frames, however, the weights are more dispersed, extending to the robot arm and the wall, as in the fourth frame of the real-robot sequence.

\begin{figure}[t]
  \centering
  \includegraphics[width=\linewidth]{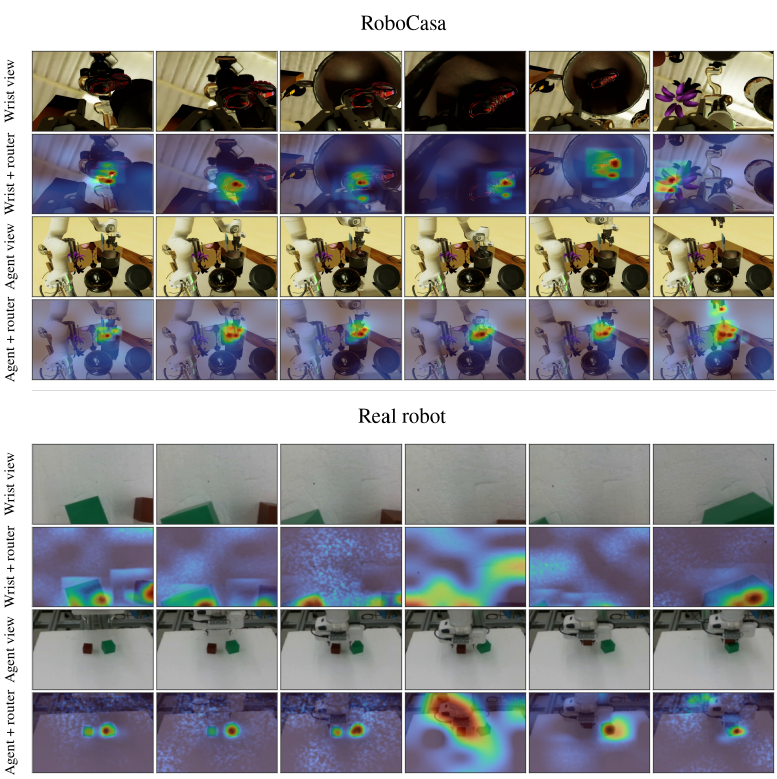}
  \caption{In most frames, the routing mass falls on the object being moved and its target. (Top) A RoboCasa365 rollout. (Bottom) The real robot stacking the red cube on the green cube. Rows show the wrist and agent (third-person) views, each followed by the same view with the routing mass overlaid. Warmer colors mark more of the top-$k$ routing mass over clean visual tokens.}
  \label{fig:router-mass-stackcube}
\end{figure}

\subsection{Routed tokens alone keep most of CWAM's success on its best tasks}
\label{sec:exp-ablations}

\looseness=-1
Routing visualizations do not establish whether the routed tokens alone are sufficient for action prediction, since the default policy retains a dense visual-to-action pathway (Section~\ref{sec:method}). They also do not show whether a different set of tokens of the same size would support comparable performance. To test these questions, we evaluate a fixed CWAM checkpoint under two visual-context interventions on the four RoboCasa365 tasks with the highest success rates. \emph{C-only} restricts the visual context available to the action queries to the routed tokens defining $\Ct$ in \eqref{eq:ct-tokens}. \emph{Random matched} instead provides the action queries with an equal number of randomly selected visual tokens.

\begin{table}[t]
  \caption{On the four RoboCasa365 tasks where CWAM succeeds most, keeping only the routed tokens costs far less success than keeping as many random tokens. The CWAM checkpoint is fixed at $50\mathrm{k}$ steps and run on the episodes of the main evaluation. The action queries see the full visual context (CWAM), only the routed tokens that form $\Ct$ (C-only), or as many visual tokens drawn uniformly at random from the clean context (random matched). Neither intervention refits the policy.}
  \label{tab:robocasa-context-ablation}
  \centering\small
  {\setlength{\tabcolsep}{5pt}%
  \resizebox{0.72\linewidth}{!}{%
  \begin{tabular}{@{}lccc@{}}
\toprule
Task & CWAM  & C-only & Random matched \\
\midrule
Pick place sink to counter & 74\% & 63\% & 28\% \\
Pick place toaster to counter & 72\% & 61\% & 21\% \\
Open drawer & 70\% & 59\% & 33\% \\
Open stand mixer head & 69\% & 58\% & 16\% \\
\midrule
Average (4 tasks) & 71.3\% & 60.3\% & 24.5\% \\
\bottomrule
\end{tabular}
  }}
\end{table}

As shown in Table~\ref{tab:robocasa-context-ablation}, \emph{C-only} retains most of the default policy's performance, whereas \emph{Random matched} results in a substantially larger drop in success rate, with the same ordering across all four tasks. These results suggest that the routed tokens retain much of the visual information needed to perform these tasks. Consistent with the synthetic-head results, the gap between the two conditions shows that token selection affects performance even when the number of retained tokens is fixed.

\vspace{-0.2em}
\section{Conclusion}
\vspace{-0.2em}

We show how to learn causal actions in WAM from observational decomposition alone. Specifically, we prove that under non-degeneracy conditions and without a sparsity constraint, each action dimension of any mechanism estimated to fit the demonstrations can contain action-irrelevant latent variables. Sparsity regularisation enables identification of the true causes of the actions and the causal structure from latent variables to actions (up to a permutation). The latent variables used to generate the actions are then an invertible function of the latent causes of the actions alone, while the full estimated state continues to describe the entire scene for prediction. Neither of the identification results rely on interventions, temporal structure or supervision. Synthetic experiments are in agreement with the theory and CWAM does in fact assist with task completion in simulation as well as on a real robot. The main limitation is the lack of component-wise identifiability that would pin down the individual latents (which may not be required for separating actions but would be useful for precise interventions). There are many more open problems in WAMs that causality can help to address and this paper represents one of the first steps towards a causal understanding of actions.

\section*{AI Statement}

In this paper, we used generative AI tools to assist coding, language editing, and figure plotting. We have reviewed all AI-assisted work. We take responsibility for the final content of this work, including text, claims or artifacts produced with the aid of generative AI.

\bibliography{references}
\bibliographystyle{iclr2027_conference}

\clearpage
\appendix

\section{Proofs}
\label{app:proof}

\paragraph{Notation.}
Table~\ref{tab:proof-notation} summarizes the symbols used in the proofs below.
The main text (Sections~\ref{sec:setup} and~\ref{sec:theory}) indexes samples by time
$t$ as in \eqref{eq:dgp-obs}--\eqref{eq:dgp-action}; the proofs fix one step, drop the
index, and write $o=f(s)$, $a=\pi(s,\eta)$. Identifiability does not use cross-time
structure, so this is equivalent to applying the argument at each $t$ separately.

\begin{table}[t]
  \centering
  \small
  \begin{tabular}{@{}p{0.26\linewidth}p{0.68\linewidth}@{}}
    \toprule
    Symbol & Meaning \\
    \midrule
    $o\in\R^{d_o}$ & Observation \\
    $s\in\R^n$ & Latent state \\
    $a\in\R^{d_a}$ & Action \\
    $f$ & Observation map $s\mapsto o$, diffeomorphism onto its image \\
    $\hat f$ & Learner observation map $\hat s\mapsto o$, diffeomorphism onto its image; its inverse $\hat f^{-1}$ is the encoder \\
    $\pi$ & Action mechanism $(s,\eta)\mapsto a$ in \eqref{eq:dgp-action} \\
    $\hat\pi$ & Estimated action mechanism \\
    $\eta$ & Exogenous action noise (held fixed in $J_\pi$, $J_{\hat\pi}$) \\
    $J_m(s)$ & Jacobian of a vector-valued map $m$ with respect to $s$ \\
    $\nz(A)$ & Index set of functionally nonzero entries of matrix or vector $A$ \\
    $\|\nz(A)\|_0$ & Cardinality of $\nz(A)$ (support size, $\ell_0$ norm) \\
    $I_i\subseteq[n]$ & Latent indices that drive action dimension $i\in[d_a]$ \\
    $\IA$ & $\bigcup_{i=1}^{d_a} I_i$; action-relevant coordinate set \\
    $s_{\IA}$ & Subvector of $s$ on coordinates $\IA$ \\
    $\hat I_i$, $\hat{I}_{\mathrm{A}}$, $C$ & Learner's index sets $\hat I_i=\{j:(i,j)\in\nz(J_{\hat\pi})\}$ and $\hat{I}_{\mathrm{A}}=\bigcup_i\hat I_i$; compact state
    $C=\hat s_{\hat{I}_{\mathrm{A}}}$ (denoted $\Ct$ in Section~\ref{sec:setup}) \\
    Observational equivalence & Same law of $(o,a)$; the results use the pathwise coupling $\hat s=\phi(s)$ with shared $\eta$, which implies it (Section~\ref{sec:setup}) \\
    $\phi$ & Reparameterization $\phi:=\hat f^{-1}\circ f$, $\hat s=\phi(s)$ \\
    $P(\hat s)$ & $J_{\phi^{-1}}(\hat s)$; chain rule links $J_{\hat\pi}$ and $J_\pi$ \\
    $\rho$ & Permutation of $[n]$ aligning true and estimated coordinates \\
    $P_\rho$ & Permutation matrix, $(P_\rho)_{j,k}=\mathbf 1\{\rho(j)=k\}$ \\
    \bottomrule
  \end{tabular}
  \caption{Summary of notation for the identifiability proofs.}
  \label{tab:proof-notation}
\end{table}

\subsection{Proof of Proposition~\ref{prop:generalist}}
\noindent\textbf{Proposition~\ref{prop:generalist}.}
Consider two observationally equivalent models $o=f(s)=\hat f(\hat s)$,
$a=\pi(s,\eta)=\hat\pi(\hat s,\eta)$ as in Table~\ref{tab:proof-notation}. Assume
that, for each action dimension $i\in[d_a]$, there exist $|I_i|$ states
$s^{(1)},\dots,s^{(|I_i|)}$ at which the rows $(J_\pi(s^{(l)}))_{i,\cdot}$ are
linearly independent, and for $\phi=\hat f^{-1}\circ f$, $P=J_{\phi^{-1}}$, and a
matrix $M$ with $\nz(M)=\nz(P)$,
$\nz\big((J_\pi(s^{(l)})M)_{i,\cdot}\big)\subseteq
\nz\big((J_{\hat\pi}(\hat s^{(l)}))_{i,\cdot}\big)$ with
$\hat s^{(l)}=\phi(s^{(l)})$, and that the regularity conditions of
Section~\ref{sec:setup} hold. Then, for every action dimension $i\in[d_a]$,
\begin{equation}
\|\nz\big((J_{\hat\pi})_{i,\cdot}\big)\|_0\;\geq\;
\|\nz\big((J_\pi)_{i,\cdot}\big)\|_0 .
\end{equation}
\par\noindent\textit{Proof.}
\medskip
\noindent\emph{Step 1 (reparameterization).}
The pathwise coupling in the hypothesis gives, for all $s$,
\begin{equation}
f(s)=\hat f(\phi(s)),\qquad
\phi:=\hat f^{-1}\circ f,\qquad \hat s=\phi(s),
\end{equation}
and $\phi$ is a diffeomorphism of $\R^n$ by the regularity conditions of
Section~\ref{sec:setup}.
For fixed $\eta$, the same coupling gives the following identity almost surely, hence for
all $s$ by continuity and full support,
\begin{equation}
\pi(s,\eta)=\hat\pi(\phi(s),\eta)\qquad\text{for all }s.
\end{equation}
Differentiating in $s$ and applying the chain rule,
\begin{equation}
\frac{\partial\pi}{\partial s}(s,\eta)
=\frac{\partial\hat\pi}{\partial\hat s}(\phi(s),\eta)\,
\frac{\partial\phi}{\partial s}(s)
=J_{\hat\pi}(\phi(s))\,J_{\phi}(s),
\end{equation}
and $J_{\phi}(s)=\big(J_{\phi^{-1}}(\phi(s))\big)^{-1}$, so with
$P(\hat s):=J_{\phi^{-1}}(\hat s)$,
\begin{equation}\label{eq:reparam}
J_{\hat\pi}(\phi(s))=J_\pi(s)\,P(\phi(s)).
\end{equation}

\medskip
\noindent\emph{Step 2 (nondegenerate rows).}
Fix $i\in[d_a]$.
Let $s^{(1)},\dots,s^{(|I_i|)}$ be as in the proposition, so the rows
$(J_\pi(s^{(l)}))_{i,\cdot}$, $l=1,\dots,|I_i|$, are linearly independent.
Write $\hat s^{(l)}=\phi(s^{(l)})$ and let $M$ be the matrix in the hypothesis, with
$\nz(M)=\nz(P)$.
The assumed inclusion is, for each $l$,
\begin{equation}\label{eq:prop-incl}
\nz\big((J_\pi(s^{(l)})M)_{i,\cdot}\big)\subseteq
\nz\big((J_{\hat\pi}(\hat s^{(l)}))_{i,\cdot}\big).
\end{equation}
Write $S_i:=\nz\big((J_{\hat\pi})_{i,\cdot}\big)$ for the functional row support
(Section~\ref{sec:setup}) and $S^{(l)}:=\nz\big((J_{\hat\pi}(\hat s^{(l)}))_{i,\cdot}\big)$.
If a row vector $v\in\R^n$ satisfies $\nz(v)\subseteq T$, then
\begin{equation}
v=\sum_{r\in T}v_r e_r\in\mathrm{span}\{e_r:r\in T\}.
\end{equation}
By \eqref{eq:prop-incl}, $\nz\big((J_\pi(s^{(l)})M)_{i,\cdot}\big)\subseteq S^{(l)}$.
Each $S^{(l)}$ is contained in $S_i$: an entry of $J_{\hat\pi}$ that is nonzero at
$\hat s^{(l)}$ is nonzero on a neighborhood of $\hat s^{(l)}$ by continuity, and this
neighborhood has positive probability because $\hat s=\phi(s)$ has a positive density on
$\R^n$ (Appendix~\ref{app:nondeg}).
Applying the display above with $v=(J_\pi(s^{(l)})M)_{i,\cdot}$ gives
\begin{equation}\label{eq:row-span}
(J_\pi(s^{(l)})M)_{i,\cdot}\in
\mathrm{span}\{e_r:r\in S_i\},\qquad l=1,\dots,|I_i|.
\end{equation}

\medskip
\noindent\emph{Step 3 (from row span to every row of $M$ on $I_i$).}
Fix $j\in I_i$.
By definition of $I_i$, coordinate $k\notin I_i$ implies
$(J_\pi(s))_{i,k}=0$ almost surely, hence everywhere by continuity and full support,
and in particular $(J_\pi(s^{(l)}))_{i,k}=0$ for every $l$.
Thus each row $v^{(l)}:=(J_\pi(s^{(l)}))_{i,\cdot}$ lies in
$V:=\mathrm{span}\{e_k:k\in I_i\}$.
The $|I_i|$ rows $v^{(1)},\dots,v^{(|I_i|)}$ are linearly independent and contained in
the $|I_i|$-dimensional space $V$, so they form a basis of $V$.
In particular $e_j^{\top}\in V$, and there exist $\beta_1,\dots,\beta_{|I_i|}$ with
\begin{equation}
e_j^{\top}=\sum_{l=1}^{|I_i|}\beta_l\,v^{(l)}
=\sum_{l=1}^{|I_i|}\beta_l\,(J_\pi(s^{(l)}))_{i,\cdot}.
\end{equation}
Right-multiplying by $M$ and using \eqref{eq:row-span},
\begin{equation}
M_{j,\cdot}=e_j^{\top}M
=\sum_{l=1}^{|I_i|}\beta_l\,(J_\pi(s^{(l)})M)_{i,\cdot}
\in\mathrm{span}\{e_r:r\in S_i\},
\end{equation}
because each term lies in $\mathrm{span}\{e_r:r\in S_i\}$ and that set is a linear
subspace. Record
\begin{equation}\label{eq:M-rows}
M_{j,\cdot}\in\mathrm{span}\{e_r:r\in S_i\}.
\end{equation}

\medskip
\noindent\emph{Step 4 (a permutation common to all action dimensions).}
Fix a point $\hat s^{*}\in\R^n$, the same for every $i$.
The matrix $P(\hat s^{*})$ is invertible, so $\det P(\hat s^{*})\neq0$.
By the Leibniz formula,
\begin{equation}
\det P(\hat s^{*})
=\sum_{\sigma\in\mathfrak S_n}\mathrm{sgn}(\sigma)\,
\prod_{k=1}^n P_{k,\sigma(k)}(\hat s^{*}),
\end{equation}
where $\mathfrak S_n$ is the set of permutations of $[n]$.
Some summand is nonzero; fix one such $\rho\in\mathfrak S_n$, chosen once for all $i$, so that
\begin{equation}
P_{j,\rho(j)}(\hat s^{*})\neq0\qquad\text{for all }j\in[n].
\end{equation}
By continuity of $P$, these entries are nonzero on a neighborhood of $\hat s^{*}$, and this
neighborhood has positive probability because $\hat s=\phi(s)$ has a positive density on $\R^n$.
Hence $(j,\rho(j))\in\nz(P)=\nz(M)$, that is,
\begin{equation}
M_{j,\rho(j)}\neq0\qquad\text{for all }j\in[n].
\end{equation}

\medskip
\noindent\emph{Step 5 (conclusion).}
Fix $j\in I_i$. Vectors in $\mathrm{span}\{e_r:r\in S_i\}$ vanish off $S_i$, so
\eqref{eq:M-rows} gives $M_{j,k}=0$ for $k\notin S_i$; in particular,
$M_{j,\rho(j)}=0$ if $\rho(j)\notin S_i$.
Step~4 gives $M_{j,\rho(j)}\neq0$, hence
\begin{equation}
\rho(j)\in S_i=\nz\big((J_{\hat\pi})_{i,\cdot}\big).
\end{equation}
This holds for every $j\in I_i$, so $\rho(I_i)\subseteq S_i$ and
\begin{equation}
\|\nz((J_{\hat\pi})_{i,\cdot})\|_0
=|S_i|\geq|\rho(I_i)|=|I_i|=\|\nz((J_\pi)_{i,\cdot})\|_0,
\end{equation}
where $|\rho(I_i)|=|I_i|$ because $\rho$ is a bijection of $[n]$.
Step~4 does not use $i$, so the same $\rho$ serves every action dimension~$i$.
\hfill$\square$

\subsection{Proof of Theorem~\ref{thm:action}}
\noindent\textbf{Theorem~\ref{thm:action}.}
Consider two observationally equivalent models as in
Proposition~\ref{prop:generalist} and assume support sparsity
\eqref{eq:sparsity}. Then there is a permutation $\rho$ of $[n]$ such that
$\nz(J_{\hat\pi})=\nz(J_\pi)P_\rho$ and
$\nz((J_{\hat\pi})_{i,\cdot})=\rho(I_i)$ for every $i\in[d_a]$.
\par\noindent\textit{Proof.}
\medskip
\noindent\emph{Step 1 (one-sided row supports).}
Apply Steps 1 to 5 of the proof of Proposition~\ref{prop:generalist} to each $i\in[d_a]$,
with the same $\hat s^{*}$ and the same $\rho$ in Step~4. This gives one permutation $\rho$
of $[n]$, common to all $i$, such that
\begin{equation}
\rho(I_i)\subseteq\nz\big((J_{\hat\pi})_{i,\cdot}\big),
\end{equation}
hence
\begin{equation}\label{eq:row-lower-i}
\|\nz((J_{\hat\pi})_{i,\cdot})\|_0\;\geq\;|\rho(I_i)|=|I_i|
=\|\nz((J_\pi)_{i,\cdot})\|_0 .
\end{equation}

\medskip
\noindent\emph{Step 2 (global one-sided bound).}
Each index pair $(i,j)\in\nz(J_\pi)$ satisfies $j\in I_i$, so
\begin{equation}
\|\nz(J_\pi)\|_0=\sum_{i=1}^{d_a}\|\nz((J_\pi)_{i,\cdot})\|_0
=\sum_{i=1}^{d_a}|I_i|.
\end{equation}
Similarly,
\begin{equation}
\|\nz(J_{\hat\pi})\|_0=\sum_{i=1}^{d_a}\|\nz((J_{\hat\pi})_{i,\cdot})\|_0.
\end{equation}
Summing \eqref{eq:row-lower-i} over $i$ gives
\begin{equation}\label{eq:row-lower}
\|\nz(J_{\hat\pi})\|_0\;\geq\;\|\nz(J_\pi)\|_0 .
\end{equation}

\medskip
\noindent\emph{Step 3 (sparsity forces equalities).}
Support sparsity \eqref{eq:sparsity} gives
\begin{equation}
\|\nz(J_{\hat\pi})\|_0\leq\|\nz(J_\pi)\|_0.
\end{equation}
Together with \eqref{eq:row-lower},
\begin{equation}
\|\nz(J_{\hat\pi})\|_0=\|\nz(J_\pi)\|_0.
\end{equation}
If some row inequality in \eqref{eq:row-lower-i} were strict, summing would make
\eqref{eq:row-lower} strict, a contradiction.
Therefore \eqref{eq:row-lower-i} holds with equality for every $i$, and with
$\rho(I_i)\subseteq\nz((J_{\hat\pi})_{i,\cdot})$ this implies
\begin{equation}\label{eq:row-eq}
\nz\big((J_{\hat\pi})_{i,\cdot}\big)=\rho(I_i)\qquad\text{for each }i\in[d_a].
\end{equation}

\medskip
\noindent\emph{Step 4 (support-level equivalence).}
For every $i\in[d_a]$ and $j\in[n]$, the definition of $I_i$, the bijectivity of
$\rho$, and \eqref{eq:row-eq} give
\begin{equation}\label{eq:iff-support}
(i,j)\in\nz(J_\pi)
\iff j\in I_i
\iff \rho(j)\in\rho(I_i)=\nz\big((J_{\hat\pi})_{i,\cdot}\big)
\iff (i,\rho(j))\in\nz(J_{\hat\pi}).
\end{equation}

\medskip
\noindent\emph{Step 5 (matrix support).}
By definition of $P_\rho$, $(i,k)\in\nz(J_\pi)P_\rho$ if and only if
$(i,\rho^{-1}(k))\in\nz(J_\pi)$.
For $k=\rho(j)$ this is equivalent to $(i,j)\in\nz(J_\pi)$, i.e.\ to
$(i,\rho(j))\in\nz(J_{\hat\pi})$ by \eqref{eq:iff-support}.
Hence
\begin{equation}
\nz(J_{\hat\pi})=\nz(J_\pi)P_\rho.
\end{equation}
\hfill$\square$

\subsection{Proof of Theorem~\ref{thm:state}}
\noindent\textbf{Theorem~\ref{thm:state}.}
Under the hypotheses of Theorem~\ref{thm:action}, with the same permutation
$\rho$, there exists an invertible map $h$ such that
$C=\hat s_{\rho(\IA)}=h(s_{\IA})$
(equivalently $\Ct=\hat s_{t,\rho(\IA)}=h(s_{t,\IA})$ in Section~\ref{sec:setup}).
\par\noindent\textit{Proof.}
Assume \eqref{eq:row-eq} and fix the permutation $\rho$ from
Theorem~\ref{thm:action}.

\medskip
\noindent\emph{Step 1 (a zero block of $P$).}
Fix any $j\in\IA$ and $r\in[n]\setminus\IA$.
Since $\IA=\bigcup_i I_i$, there exists an action dimension
$i\in[d_a]$ such that $j\in I_i$.
Equation~\ref{eq:M-rows} from the proof of
Proposition~\ref{prop:generalist}, together with
\eqref{eq:row-eq}, gives
\[
M_{j,\cdot}
\in \mathrm{span}\{e_k:k\in S_i\}
=
\mathrm{span}\{e_k:k\in\rho(I_i)\},
\]
where $S_i=\nz((J_{\hat\pi})_{i,\cdot})$.
Because $r\notin\IA$, we have $r\notin I_i$ and hence
$\rho(r)\notin\rho(I_i)$.
Every vector in the displayed span vanishes outside
$\rho(I_i)$, so
\[
M_{j,\rho(r)}=0.
\]
By the assumed equality of functional supports
$\nz(M)=\nz(P)$, this implies that
$P_{j,\rho(r)}=0$ almost surely.
Since $P=J_{\phi^{-1}}$ is continuous and $\hat s$ has a
density positive on $\R^n$, the entry must vanish everywhere:
otherwise, continuity would make it nonzero on a neighborhood
of positive probability.
As $j$ and $r$ were arbitrary, we obtain
\begin{equation}\label{eq:block-zero}
P_{j,\rho(r)}(\hat s)=0
\qquad
\text{for all }\hat s\in\R^n,\quad
j\in\IA,\quad r\in[n]\setminus\IA.
\end{equation}

\medskip
\noindent\emph{Step 2 (block triangular form).}
Order rows of $P$ as $(\IA,[n]\setminus\IA)$ and columns as
$(\rho(\IA),\rho([n]\setminus\IA))$.
Then \eqref{eq:block-zero} is exactly $P_{j,\rho(r)}=0$ for $j\in\IA$ and
$r\in[n]\setminus\IA$, i.e.\ the upper-right block of $P$ vanishes:
\begin{equation}
P=\begin{pmatrix}P_{11}&0\\ P_{21}&P_{22}\end{pmatrix},
\qquad
P_{11}\in\R^{|\IA|\times|\rho(\IA)|},\quad
P_{22}\in\R^{(n-|\IA|)\times(n-|\rho(\IA)|)}.
\end{equation}
Because $P$ is invertible,
\begin{equation}
\det P=\det P_{11}\det P_{22}\neq0,
\end{equation}
so $P_{11}$ and $P_{22}$ are invertible.
For an invertible block-triangular matrix,
\begin{equation}\label{eq:Pinv-block}
P^{-1}=\begin{pmatrix}P_{11}^{-1}&0\\
-P_{22}^{-1}P_{21}P_{11}^{-1}&P_{22}^{-1}\end{pmatrix},
\end{equation}
which is checked by direct multiplication.
Since $P^{-1}=J_{\phi}(s)$, the upper-right block of $P^{-1}$ is zero, so
\begin{equation}
\frac{\partial\hat s_{\rho(j)}}{\partial s_r}=0
\qquad\text{for }j\in\IA,\ r\in[n]\setminus\IA;
\end{equation}
\eqref{eq:block-zero} gives
\begin{equation}
\frac{\partial s_j}{\partial\hat s_{\rho(r)}}=0
\qquad\text{for the same index pairs.}
\end{equation}

\medskip
\noindent\emph{Step 3 (block-wise functional dependence).}
Because $\phi$ and $\phi^{-1}$ are diffeomorphisms on $\R^n$, if
$\partial\hat s_{\rho(j)}/\partial s_r\equiv0$ for $j\in\IA$ and $r\in[n]\setminus\IA$,
then $\hat s_{\rho(j)}$ does not depend on $s_r$ as a function on $\R^n$; likewise
$\partial s_j/\partial\hat s_{\rho(r)}\equiv0$ implies $s_j$ does not depend on
$\hat s_{\rho(r)}$.
Thus there exist continuously differentiable maps $h$ and $\tilde h$ such that
\begin{equation}
\hat s_{\rho(\IA)}=h(s_{\IA}),\qquad s_{\IA}=\tilde h(\hat s_{\rho(\IA)}).
\end{equation}
On the image of $h$, $\tilde h\circ h$ is the identity on $s_{\IA}$ and
$h\circ\tilde h$ is the identity on $\hat s_{\rho(\IA)}$, so $h$ is invertible with
inverse $\tilde h$.
The Jacobian of $h$ is the upper-left block of $P^{-1}$, hence
\begin{equation}
J_h=P_{11}^{-1}.
\end{equation}

\medskip
\noindent\emph{Step 4 (compact state).}
By \eqref{eq:row-eq}, $\hat I_i=\nz((J_{\hat\pi})_{i,\cdot})=\rho(I_i)$ for every $i$, so $\hat{I}_{\mathrm{A}}=\bigcup_i\rho(I_i)=\rho(\IA)$ because $\rho$ is a bijection. Step~3 therefore gives
\begin{equation}
C=\hat s_{\hat{I}_{\mathrm{A}}}=\hat s_{\rho(\IA)}=h(s_{\IA}),
\end{equation}
which is Theorem~\ref{thm:state} with time subscript $t$ restored ($\Ct=\hat s_{t,\rho(\IA)}$).
\hfill$\square$

\section{Supplementary discussions}

\subsection{Detailed discussion on main conditions}
\label{app:nondeg}
The assumptions are those in Proposition~\ref{prop:generalist}. They require that
each action dimension $i$ has enough variation in $(J_\pi(s))_{i,\cdot}$ across
$|I_i|$ states to span the true parents $I_i$, and that reparameterization through
$\phi$ cannot hide a true dependency in $\hat\pi$ by cancellation at every such
state. Both parts are mild in the usual smooth, nonparametric regime.

\emph{Span.}
For each action dimension $i$, the assumption requires
$|I_i|$ states at which the rows
$(J_\pi(s^{(l)}))_{i,\cdot}$ are linearly independent.
Equivalently, these rows must span
$\mathrm{span}\{e_j:j\in I_i\}$.
This requirement concerns variation in the mechanism's
derivatives and does not follow from smoothness or from
a positive state density.
For example, if $\pi_i(s)=s_1+s_2$, then
$I_i=\{1,2\}$ but every Jacobian row equals
$(1,1,0,\ldots,0)$, so the condition fails regardless
of the state distribution.
More generally, mechanisms of the form
$\pi_i(s)=g(s_1+s_2)$ have collinear Jacobian rows.
The proposition therefore assumes the required spanning
states explicitly.
It requires their existence, rather than asserting that
$|I_i|$ randomly sampled states satisfy the condition
with probability one.

\emph{Support inclusion.} The identity \eqref{eq:reparam} implies that rows of
$J_\pi(s)M$ with $\nz(M)=\nz(P)$ typically lie inside the corresponding rows of
$J_{\hat\pi}(\hat s)$. The inclusion assumed in Proposition~\ref{prop:generalist}
must hold at the $|I_i|$ states used for the span argument; it excludes
measure-zero coincidences where partial derivatives cancel on all of those states
simultaneously.

\emph{Functional row support.} In the proof of Proposition~\ref{prop:generalist},
$S^{(l)}:=\nz\big((J_{\hat\pi}(\hat s^{(l)}))_{i,\cdot}\big)$ is compared with the
functional row support $S_i=\nz\big((J_{\hat\pi})_{i,\cdot}\big)$ from
Section~\ref{sec:setup}. If $(J_{\hat\pi}(\hat s^{(l)}))_{i,k}\neq 0$, then, since $\hat\pi$
is continuously differentiable, $J_{\hat\pi}(\hat s)_{i,k}\neq 0$ on a neighborhood of
$\hat s^{(l)}$, so $k\in S_i$ because $\hat s=\phi(s)$ has a positive density on $\R^n$
(Section~\ref{sec:setup}).

\section{Evaluation protocol}
\label{app:protocol}

\paragraph{Estimands.} Block recovery is assessed with a ridge probe from $\Ct$ to the true latent variables of the synthetic process, fitted and reported on disjoint episodes, and reported as held-out $R^2$ for the action-relevant and for the nuisance variables. The same probe from the true $s_{t,\IA}$ to the nuisance variables gives the oracle reference. On the robot benchmarks, routing mass is visualized in Figure~\ref{fig:router-mass-stackcube}, and on four RoboCasa365 tasks, closed-loop success with only the routed tokens is measured in Table~\ref{tab:robocasa-context-ablation}.

\paragraph{Synthetic generator.}
The simulator implements \eqref{eq:dgp-obs}--\eqref{eq:dgp-action} and adds
a latent transition and a hidden regime that scales the action mechanism and the transition, to stress nuisance coupling;
those extras are not part of the identification assumptions in
Section~\ref{sec:setup}. Latent $s_t\in\R^{20}$ with $\IA=\{1,\dots,4\}$. A hidden regime
$c_t\in\{0,1,2\}$ follows a Markov chain with Dirichlet-random transition matrix
and initial law. Action $a_t=M_{c_t}\,s_{t,\IA}+0.05\,\eta_t$, where
$M_c=(1+0.15(c-1))M$ with Gaussian $M\in\R^{8\times4}$; transition
$s_{t+1}=\tanh\big((1+0.15\,c_t)(W_s s_t+W_a a_t)\big)+0.05\,\epsilon_{t+1}$ with a
dense Gaussian $W_s$ (diagonal $0.85$) and $W_a$ feeding the action back into
$\IA$ only, so nuisance factors are coupled to relevant ones through $W_s$;
observation $o_t=\tanh(W_{\mathrm{mix}}s_t)+0.05\,\xi_t$ with block-diagonal Gaussian
$W_{\mathrm{mix}}\in\R^{100\times20}$ (five observation coordinates per latent factor,
no cross-block entries). The environment suite draws $e_t\in[E]$
i.i.d.\ at every step and multiplies the nuisance coordinates of $s_{t+1}$ by a
per-environment factor in $[0.5,1.5]$; the action family is unchanged across $E$.
The regime thus modulates behavior and transition without changing the four-factor
action support. It is not supplied to the estimator, and the fitted single-frame head
learns a marginalized predictor rather than an explicitly recovered conditional policy.
Simulator streams are seeded so that exogenous noise is paired across $E$.

\paragraph{Synthetic estimators and metrics.}
The encoder is an MLP $100\to256\to20$ with $\tanh$ followed by a non-affine LayerNorm,
the action head is linear $20\to8$, and we use Adam with batch size $512$. Development seeds $900,901$
select learning rate from $\{3\cdot10^{-4},10^{-3},3\cdot10^{-3}\}$ and epochs
from $\{80,160\}$, with $320$ epochs tried only if no candidate meets the
predeclared development targets.
All six candidates meet the development-mean targets; minimum selected-state
action MSE selects learning rate $10^{-3}$ and $160$ epochs. This setting is shared
by all baselines and evaluated once on seeds $200$--$204$, held out from development tuning.
The current action is
a target, never an encoder input. Dense, random and $\lambda{=}0$ readouts share
weights, and $E=1$ is reused across tables: $70$ entries come from $50$ saved
estimators, all reloaded for metric replay.
CWAM penalty: mean over latent coordinates of the $\ell_2$ norm of the
corresponding head column, weight $0.01$. VAE (Table~\ref{tab:synthetic}): Gaussian
posterior, MLP decoder, KL weight $10^{-3}$, head sparsity weight $0$; an optional
VAE~+~$\ell_1$ baseline ($0.05$ times the mean absolute head weight) is trained in
the same matrix but omitted from the table.
Support: the four coordinates with the largest mean absolute head weight; random
matched: four uniformly random coordinates. Probes: ridge ($\alpha=1$) fitted on
$250$ calibration episodes, $R^2$ averaged over target dimensions on $250$
disjoint reporting episodes. Sel./full and Compl./full apply the fitted head to
$\hat s_t$ with the complement (resp.\ $\Ct$) set to zero on the reporting
episodes.

\paragraph{Synthetic splits, readouts, and the oracle reference.} Each fit uses $2{,}000$ training episodes and $500$ held-out episodes of length $128$. The hidden regime $c_t$ enters the action mechanism and the transition in every reported run. The dense readout applies the probes to all $20$ coordinates of $\hat s_t$. Random matched and CWAM without sparsity are taken from the same fit, trained with head sparsity weight $\lambda=0$, and CWAM without sparsity keeps its four coordinates with the largest mean absolute head weight. CWAM with sparsity and the standard MLP VAE are separate fits, and $\Ct$ is selected from the head of the VAE in the same way. Every top-$4$ selection uses the true size $|\IA|=4$. Because the transition couples the nuisance variables to the action-relevant ones, the true block already predicts the nuisance variables with $R^2>0$. The oracle reference $R^2(s_{t,\IA}\to\mathrm{nuis})$ measures this coupling with the same linear probe and splits, and $\Delta_{\mathrm{nuis}}$ subtracts it from the nuisance $R^2$ of $\Ct$ seed by seed before averaging. The reference is specific to this simulator and is computed from the true coordinates $s_{t,\IA}$. Since $\Ct=\hat s_{t,\hat{I}_{\mathrm{A}}}$ passes through a nonlinear encoder and need not align with $s_{t,\IA}$ coordinate by coordinate, a linear probe from $\Ct$ can land slightly above or below the reference even when block recovery is good, so $\Delta_{\mathrm{nuis}}$ can be slightly negative. The environment suite uses $E\in\{1,2,3,5\}$, and neither $c_t$ nor $e_t$ is an input to the encoder. Tables report the mean and the population standard deviation over seeds, with ratios computed per seed. Theorem~\ref{thm:state} identifies the action-relevant latent variables up to a block-wise invertible transformation, so we measure block-level recovery with linear-probe $R^2$ and do not report the MCC.

\section{Hidden environments (synthetic)}
\label{app:synthetic-env}

The environment index $e_t\in[E]$ in Appendix~\ref{app:protocol} rescales nuisance
transition coordinates while the action mechanism is fixed. Table~\ref{tab:synthetic-env}
reports probes of the dense readout and of CWAM with sparsity as $E$ grows (mean $\pm$ std over five
seeds); metrics match Table~\ref{tab:synthetic}.

\begin{table}[t]
  \caption{Block recovery by CWAM with sparsity does not degrade as hidden environments are added. At every $E$ it decodes the action-relevant block as well as the dense readout, which alone decodes nuisance variables beyond the true block. Each hidden environment rescales the nuisance variables in the transition, with the action mechanism fixed across $E$, and is never an encoder input. Metrics as in Table~\ref{tab:synthetic}, whose fits give the row $E=1$. Mean $\pm$ population std over five seeds.}
  \label{tab:synthetic-env}
  \centering
  {\footnotesize\setlength{\tabcolsep}{3.5pt}%
  \resizebox{0.75\linewidth}{!}{%
  \begin{tabular}{@{}c@{\hspace{0.4em}}cccc@{}}
\toprule
 & \multicolumn{2}{c}{Dense (20D readout)} & \multicolumn{2}{c}{CWAM with sparsity (4D)} \\
\cmidrule(lr){2-3}\cmidrule(lr){4-5}
$E$ & Relevant $R^2\uparrow$ & $\Delta_{\mathrm{nuis}}\downarrow$ & Relevant $R^2\uparrow$ & $\Delta_{\mathrm{nuis}}\downarrow$ \\
\midrule
1 & 0.9961 $\pm$ 0.0015 & 0.646 $\pm$ 0.115 & 0.9957 $\pm$ 0.0014 & 0.039 $\pm$ 0.093 \\
2 & 0.9961 $\pm$ 0.0016 & 0.640 $\pm$ 0.110 & 0.9961 $\pm$ 0.0016 & -0.004 $\pm$ 0.003 \\
3 & 0.9961 $\pm$ 0.0017 & 0.677 $\pm$ 0.109 & 0.9961 $\pm$ 0.0017 & -0.004 $\pm$ 0.003 \\
5 & 0.9961 $\pm$ 0.0017 & 0.683 $\pm$ 0.114 & 0.9961 $\pm$ 0.0017 & -0.003 $\pm$ 0.002 \\
\bottomrule
\end{tabular}

  }}
\end{table}

\section{RoboCasa365 per-task success}
\label{app:learned-success}

The RoboCasa365 panel of Figure~\ref{fig:robocasa-libero-radar} reports the mean success within each foundational skill family over the $17$ evaluation tasks. Table~\ref{tab:robocasa-per-task} lists the per-task success rates behind those family means, together with the macro average over all tasks.

\paragraph{Robot evaluation.} The matched Cosmos~3 baseline and CWAM are post-trained on the same demonstrations with the same recipe and optimization, and CWAM differs only by the action-relevance router and $\mathcal{L}_{\mathrm{sparse}}$ (Section~\ref{sec:method}), so the baseline has no router. RoboCasa365 is evaluated on $17$ tasks with $100$ episodes per task, LIBERO on $40$ tasks in four suites with $100$ episodes per task, and the Franka stack-cube task with $50$ trials per policy. For the context ablation in Table~\ref{tab:robocasa-context-ablation}, the CWAM checkpoint is fixed, and only the visual tokens that the action queries attend to are changed at evaluation, with no refitting and with the episode keys of the main evaluation. C-only keeps the routed top-$k$ tokens that define $\Ct$ in \eqref{eq:ct-tokens} and zeroes the other visual context to the action head. Random matched keeps the same number of visual tokens, drawn uniformly at random from the clean context. The four ablation tasks are those with the highest CWAM success in Table~\ref{tab:robocasa-per-task}.

\begin{table}[t]
  \caption{CWAM raises success on $13$ of the $17$ RoboCasa365 tasks and the macro average from $0.24$ to $0.43$. Each cell is the success rate over $100$ episodes. Tasks are grouped by skill family, and each point in Figure~\ref{fig:robocasa-libero-radar}(a) averages the tasks of one family. Baseline and CWAM are as in Figure~\ref{fig:robocasa-libero-radar}.}
  \label{tab:robocasa-per-task}
  \centering\small
  \setlength{\tabcolsep}{6pt}
  \newcommand{\robopct}[1]{#1}

\begin{tabular}{@{}llcc@{}}
  \toprule
  \textbf{Foundational skill family} & \textbf{Task} & \textbf{Baseline} & \textbf{CWAM} \\
  \midrule
  Pick and place
  & PickPlaceCtrToCab & \robopct{.31} & \robopct{.43} \\
  & PickPlaceCtrToStove & \robopct{.22} & \robopct{.45} \\
  & PickPlaceDrawerToCtr & \robopct{.15} & \robopct{.07} \\
  & PickPlaceSinkToCtr & \robopct{.24} & \robopct{.74} \\
  & PickPlaceToasterToCtr & \robopct{.28} & \robopct{.72} \\
  \addlinespace
  Opening / closing doors
  & CloseFridge & \robopct{.14} & \robopct{.61} \\
  & CloseToasterOvenDoor & \robopct{.51} & \robopct{.37} \\
  & OpenCabinet & \robopct{.31} & \robopct{.24} \\
  \addlinespace
  Closing / opening lids
  & CloseBlenderLid & \robopct{.02} & \robopct{.03} \\
  & OpenStandMixerHead & \robopct{.30} & \robopct{.69} \\
  \addlinespace
  Pressing buttons
  & TurnOnMicrowave & \robopct{.27} & \robopct{.26} \\
  & TurnOnElectricKettle & \robopct{.35} & \robopct{.45} \\
  \addlinespace
  Opening / closing drawers
  & OpenDrawer & \robopct{.25} & \robopct{.70} \\
  \addlinespace
  Sliding racks
  & SlideDishwasherRack & \robopct{.32} & \robopct{.61} \\
  \addlinespace
  Twisting knobs
  & TurnOffStove & \robopct{.05} & \robopct{.18} \\
  \addlinespace
  Turning levers
  & TurnOnSinkFaucet & \robopct{.25} & \robopct{.37} \\
  \addlinespace
  Insertion
  & CoffeeSetupMug & \robopct{.10} & \robopct{.37} \\
  \midrule
  \multicolumn{2}{@{}l}{\textbf{Macro (unweighted mean over $17$ tasks)}}
  & \robopct{.24} & \robopct{.43} \\
  \bottomrule
\end{tabular}

\end{table}

\section{Cosmos~3 post-training setup}
\label{app:posttrain}

This appendix records the matched baseline and CWAM post-training recipe used for
Section~\ref{sec:exp-main-learned}. The architectural difference is the
action-relevance router and $\mathcal{L}_{\mathrm{sparse}}$ (Appendix~\ref{app:posttrain-router});
optimization, data, packing, and rectified-flow losses are otherwise identical
(Appendix~\ref{app:posttrain-base}).

\subsection{Training stages and served checkpoint}
\label{app:posttrain-stages}

We warm-start from the public \texttt{nvidia/Cosmos3-Nano} checkpoint converted to
distributed checkpoint (DCP) format.

\textbf{Stage~1 (router initialization).}
CWAM post-trains on RoboCasa365 for $10{,}000$ optimizer steps with the
action-relevance router enabled and freshly initialized; the rectified-flow
objective, data packing, and shared hyperparameters match
Table~\ref{tab:posttrain-opt} and Appendix~\ref{app:posttrain-base}.

\textbf{Stage~2 (continue).}
Training resumes from step $10{,}000$ and runs to $50{,}000$ steps with a flat
LR multiplier $0.3871$ (the learning rate reached at the end of Stage~1;
Table~\ref{tab:posttrain-opt}); optimizer scheduler state is not resumed.
The matched baseline follows the same two stages and data recipe without the router module (hence no $\mathcal{L}_{\mathrm{sparse}}$).

Training runs on one $8{\times}$H200 node (FSDP shard degree $8$), bf16, selective
activation checkpointing on FMHA, and precomputed Wan~2.2 video latents bundled
with the RoboCasa365 training shards.

\subsection{Demonstration data recipe (RoboCasa365)}
\label{app:posttrain-data}

\paragraph{Corpus.}
Processed LeRobot-style RoboCasa365 demonstrations from the Cosmos~3
pipeline.

\paragraph{Windowing and modalities.}
Policy mode at $20$\,Hz: \texttt{chunk\_length=16} frames per training window; the
leading frame is observed and later video tokens are noised under the rectified-flow
objective. Observations use \texttt{camera\_mode=concat\_view} (third-person and
wrist, each $256{\times}256$, concatenated to $256{\times}512$ before the model
canvas). Language follows \texttt{caption\_mode=policy} with JSON-formatted prompts.
Actions are \texttt{frame\_wise\_relative} rot6d (10D: position, rotation, gripper)
with \texttt{quantile\_rot} normalization and embodiment tag
\texttt{robocasa365}. Classifier-free guidance dropout on conditioning is
$0.1$ on train and $0$ on val.

\subsection{Closed-loop task protocols}
\label{app:posttrain-eval}

\paragraph{RoboCasa365 (17 tasks).}
We evaluate the \emph{target} split on the $17$ tasks listed in
Table~\ref{tab:robocasa-per-task}:
\texttt{CloseBlenderLid},\texttt{CloseFridge}, \texttt{CloseToasterOvenDoor}, \texttt{CoffeeSetupMug}, \texttt{OpenCabinet}, \texttt{OpenDrawer}, \texttt{OpenStandMixerHead}, \texttt{PickPlaceCounterToCabinet}, \texttt{PickPlaceCounterToStove}, \texttt{PickPlaceDrawerToCounter}, \texttt{PickPlaceSinkToCounter}, \texttt{PickPlaceToasterToCounter}, \texttt{SlideDishwasherRack}, \texttt{TurnOffStove}, \texttt{TurnOnElectricKettle}, \texttt{TurnOnMicrowave}, \texttt{TurnOnSinkFaucet}. Each task: \textbf{$100$ episodes}, base seed
\textbf{$42$} (paired keys across intervention conditions in
Appendix~\ref{app:learned-success}), \textbf{$400$} environment steps cap,
\textbf{$12$}-step open-loop action horizon (replan after each chunk).

\paragraph{LIBERO (40 tasks).}
Closed-loop evaluation uses the \textbf{LIBERO-post-trained} baseline and CWAM
checkpoint (Appendix~\ref{app:posttrain-data}).
All LIBERO tasks are grouped into the four standard suites ($40$ tasks total);
per-suite success rates appear in Figure~\ref{fig:robocasa-libero-radar}(b).
Protocol details follow the Cosmos~3 LIBERO closed-loop stack (20\,Hz, quantile
rotation actions, concatenated views).

\subsection{Shared optimization hyperparameters}
\label{app:posttrain-opt}

Table~\ref{tab:posttrain-opt} lists settings shared by baseline and CWAM; router-only
fields are deferred to Appendix~\ref{app:posttrain-router}.

\begin{table}[t]
  \caption{Shared Cosmos~3 RoboCasa365 post-training hyperparameters (baseline and CWAM).}
  \label{tab:posttrain-opt}
  \centering\small
  \begin{tabular}{@{}ll@{}}
    \toprule
    \textbf{Setting} & \textbf{Value} \\
    \midrule
    Optimizer & FusedAdamW ($\beta_1{=}0.9$, $\beta_2{=}0.99$, $\eps{=}10^{-8}$) \\
    Base learning rate & $5\times 10^{-5}$ \\
    LR multipliers & $5.0$ on \texttt{action2llm}, \texttt{llm2action}, \texttt{action\_modality\_embed} \\
    Weight decay & $0.05$ \\
    Gradient clipping & $1.0$ (force finite) \\
    Precision & bf16 \\
    Parallelism & FSDP, 8 GPU shard / 1 replicate \\
    Grad accumulation & $2$ (Stages~1 and~2) \\
    Stage~1 schedule & linear warmup $500$ steps, cycle $16{,}000$, \texttt{max\_iter=10{,}000} \\
    Stage~2 schedule & flat multiplier $0.3871$ (LR at step $10{,}000$), cycle $90{,}000$, \texttt{max\_iter=50{,}000} \\
    Checkpointing & save every $500$ (Stage~1) / $10{,}000$ (Stage~2); EMA rate $0.1$ \\
    Rectified-flow loss scale & $10.0$ on generator (RoboCasa recipe override) \\
    Action loss weight & $10.0$ \\
    CFG dropout (train) & $0.1$ \\
    \bottomrule
  \end{tabular}
\end{table}

\section{Cosmos~3 base model and rectified-flow objective}
\label{app:posttrain-base}

\paragraph{Backbone.}
Cosmos~3-Nano: Qwen3-VL-8B-Instruct reasoner inside a mixture-of-transformers (MoT)
generator with \texttt{joint\_attn\_implementation=two\_way}, video generation enabled
(\texttt{vision\_gen=True}), and action generation enabled (\texttt{action\_gen=True}).
Diffusion expert loads pretrained weights; RoboCasa365 post-training \emph{reinitializes}
action adapters (\texttt{action2llm}, \texttt{llm2action}, \texttt{action\_modality\_embed},
\texttt{action\_pos\_embed}) from scratch while keeping the public video/reasoner stack.

\paragraph{Tokenization.}
Text: staged local Qwen3-VL processor (\texttt{QWEN3\_VL\_TOKENIZER\_PATH}). Video:
Wan~2.2 VAE latents ($48$ channels) with training-time exact encode durations
$\{17,61,73\}$ frames when an online VAE is used; production runs use precomputed
latents with metadata-only validation. Max packed tokens after packing:
$74{,}000$ for this recipe (override of the Nano default $45{,}056$).

\paragraph{Objective $\mathcal{L}_{\mathrm{WAM}}$.}
Joint rectified-flow training on noised future video tokens and the executed action
chunk, conditioned on clean history and language. Training-time distributions:
\texttt{waver} on video, \texttt{logitnormal} on image/action/sound branches;
\texttt{loss\_scale=10.0} on the generator path in the RoboCasa365 config;
\texttt{action\_loss\_weight=10.0}. Inference defaults (\texttt{rectified\_flow\_inference\_config}):
UniPC, $1000$ train timesteps, shift $1$. The baseline and CWAM share this loss;
CWAM adds $\mathcal{L}_{\mathrm{sparse}}$ only on the router readout
(Section~\ref{sec:method}).

\section{Action-relevance router}
\label{app:posttrain-router}

\paragraph{Module.}
When \texttt{action\_relevance\_enabled=True} , a residual router sits after the
generator tower: relation-aware attention over clean context indices
$\mathcal{J}_t$ (task + clean vision; noised future vision excluded), with routing
dimension $d_r{=}512$, $8$ heads, top-$k$ support $k{=}32$, and residual gate
initialized to $0.01$. Optimizer trains \texttt{action\_relevance} with LR multiplier
$5.0$; router weights are skipped when loading the public Nano checkpoint.

\paragraph{Sparsity schedule.}
$\lambda{=}0.01$ on $\mathcal{L}_{\mathrm{sparse}}$ (\eqref{eq:router-sparse}); not applicable to the matched baseline, which has no router. Sparsity loss activates at optimizer step
$1{,}000$ and anneals linearly over $4{,}000$ steps (matches the main-text ramp).
Gradients from $\mathcal{L}_{\mathrm{sparse}}$ update the shared tower and router;
the dense MoT visual$\to$action path remains active at training and default eval time
(Section~\ref{sec:method}).

\paragraph{Evaluation-only readouts.}
Closed-loop \texttt{selected} (C-only) and \texttt{random\_matched} conditions
(Section~\ref{sec:exp-ablations}, Appendix~\ref{app:learned-success}) intervene on the
\emph{fixed} $50\mathrm{k}$ policy without refitting; they are not alternate training
configs.
\end{document}